\documentclass[10 pt, conference]{ieeeconf}  
\IEEEoverridecommandlockouts                              
\usepackage{graphicx} 
\usepackage{amssymb,mathtools,mathrsfs}  
\usepackage{algorithm}
\usepackage{parskip}
\usepackage{xcolor}
\usepackage[noend]{algpseudocode}
\newtheorem{theorem}{Theorem}

\newtheorem{proposition}{Proposition}	
	
\newtheorem{definition}{Definition}

\title{\LARGE \bf
Contact-Aware Controller Design for Complementarity Systems
}

\author{Alp Aydinoglu$^{1}$, Victor M. Preciado$^{1}$, Michael Posa$^{1}$
\thanks{*This work was supported by the National Science Foundation under Grant No. CMMI-1830218.
}
    \thanks{$^{1}$Alp Aydinoglu and Victor M. Preciado are with the Department of Electrical and Systems Engineering, and Michael Posa is with the General Robotics, Automation, Sensing and Perception (GRASP) Laboratory, University of Pennsylvania, Philadelphia, PA 19104, USA.
    {\tt\small \{alpayd, preciado, posa\}@seas.upenn.edu}}%
}

\begin{document}
\maketitle
\thispagestyle{empty}
\pagestyle{empty}


\begin{abstract}
While many robotic tasks, like manipulation and locomotion, are fundamentally based in making and breaking contact with the environment, state-of-the-art control policies struggle to deal with the hybrid nature of multi-contact motion.
Such controllers often rely heavily upon heuristics or, due to the combinatoric structure in the dynamics, are unsuitable for real-time control.
Principled deployment of tactile sensors offers a promising mechanism for stable and robust control, but modern approaches often use this data in an ad hoc manner, for instance to guide guarded moves.
In this work, by exploiting the complementarity structure of contact dynamics, we propose a control framework which can close the loop on rich, tactile sensors.
Critically, this framework is non-combinatoric, enabling optimization algorithms to automatically synthesize provably stable control policies.
We demonstrate this approach on three different underactuated, multi-contact robotics problems.
\end{abstract}

\section{Introduction}
In recent years, robotic automation has excelled in dealing with repetitive tasks in static and structured environments.
On the other hand, to achieve the promise of the field, robots must perform efficiently in complex, unstructured environments which involve physical interaction between the robot and the environment itself.
Furthermore, as compared with traditional motion planning problems, tasks like dexterous manipulation and legged locomotion fundamentally require intentionally initiating contact with the environment to achieve a positive result.
To enable stable, and robust motion, it is critically important to design policies that explicitly consider the interaction between robot and environment.

Contact, however, is hybrid or multi-modal in nature, capturing the effect of stick-slip transitions or making and breaking contact.
Standard approaches to control often match the hybrid dynamics with a hybrid or switching controller, where one policy is associated with each mode.
However, precise identification of the hybrid events is difficult in practice, and switching controllers can be brittle, particularly local to the switching surface, or require significant hand-tuning.
Model predictive control, closely related to this work, is one approach that has been regularly applied to control through contact, with notable successes.
Due to the computational complexity of hybrid model predictive control, these approaches must either approximate the hybrid dynamics (e.g. \cite{Tassa10}), limit online control to a known mode sequence \cite{Hogan2017}, or are unable to perform in real-time \cite{marcucci2017approximate}.
While prior work has explored computational synthesis of non-switching feedback policies \cite{posa2015stability}, it does not incorporate tactile sensing.
However, there are clear structural limits to smooth, state-based control.
Here, we focus on offline synthesis of a stabilizing feedback policy, eliminating the need for intensive online calculations.

The need for contact-aware control is driven, in part, by recent advances in tactile sensing (e.g \cite{wettels2009multi, guggenheim2017robust, guggenheim2017robust, yuan2015measurement, kumar2016optimal, donlon2018gelslim} and others).
Given these advances, there has been ongoing research to design control policies using tactile feedback for tasks that require making and breaking contact. 
However, these approaches are largely based on static assumptions, for instance with guarded moves \cite{howe1993tactile}, or rely upon switching controllers (e.g. \cite{romano2011human, yamaguchi2016combining}).
Other recent methods  incorporate tactile sensors within deep learning frameworks, though offer no guarantees on performance or stability  \cite{merzic2018leveraging, tian2019manipulation}.

In this work, we present an optimization-based numerical approach for designing control policies that use feedback based on the contact forces. The control policy combines regular state feedback with tactile feedback in order to provably stabilize the system.
Our controller design is non-combinatoric in nature and avoids enumerating the exponential number of potential hybrid modes that might arise from contact.
More precisely, we design a piecewise affine controller where the contributions of each contact are additive (rather than combinatoric) in nature.
Inspired by both prior work \cite{posa2015stability} and \cite{camlibel2001complementarity}, we synthesize and verify a corresponding non-smooth, piecewise quadratic Lyapunov function.
Furthermore, we also consider scenarios where full state information might be lacking, such as when vision of an object is occluded but tactile information is available.
To address, this, we synthesize output feedback controllers, here expressed via specific sparsity patterns in the state feedback matrix.

The primary contribution of this paper is an algorithm for synthesis of a control policy, utilizing state and force feedback, which is provably stabilizing even during contact mode transitions.
We choose a structure for controller and Lyapunov function designed specifically to leverage the complementarity structure of contact, enabling scaling to multi-contact control.
This problem is formulated and solved as a bilinear matrix inequality (BMI).
\section{Background}
We now introduce the notation that is used throughout this work. The directional derivative of a function $z(x)$ in the direction $d \in \mathbb{R}^n$ is given as follows,
\begin{equation*}
    z'(x;d) = \lim_{\tau \downarrow  0} \frac{z (x + \tau d) - z(x)}{\tau}.
\end{equation*}
For a positive integer $l$, $\bar{l}$ denotes the set $\{ 1, 2, \ldots, l \}$. Given a matrix $M \in \mathbb{R}^{k \times l}$ and two subsets $I \subseteq \bar{k}$ and $J \subseteq \bar{l}$, we define $M_{I J} = (m_{i j})_{i \in I, j \in J}$. For the case where $J = \bar{l}$, we use the shorthand notation $M_{I \bullet}$. For two vectors $a \in \mathbb{R}^m$ and $b \in \mathbb{R}^m$, we use the notation $0 \leq a \perp b \geq 0$ to denote that $a \geq 0, \; b \geq 0, \; a^T b = 0$.

\subsection{Linear Complementarity Systems}
A standard approach to modeling robotic systems is through the framework of rigid-body systems with contacts. The continuous time dynamics can be modeled by the manipulator equations
\begin{equation}
\label{eq:manipulator}
    M(q) \dot{v} + C(q,v) = Bu + J(q)^T \lambda,
\end{equation}
where $q \in \mathbb{R}^{p}$ represents the generalized coordinates, $v \in \mathbb{R}^{p}$ represents the generalized velocities, $\lambda \in \mathbb{R}^m$ represents the contact forces, $M(q)$ is the inertia matrix, $C(q,v)$ represents the combined Coriolis and gravitational terms, $B$ maps the control inputs $u$ into joint coordinates and $J(q)$ is the projection matrix. 

The model \eqref{eq:manipulator} is a hybrid dynamical system \cite{alur2000discrete}, \cite{branicky1998unified} with $2^m$ modes that arise from distinct combinations of contacts. 
One approach to contact dynamics describes the forces using the complementarity framework where the generalized coordinates $q$ and contact forces $\lambda$ satisfy a set of complementarity constraints \cite{cottle2009linear}:
\begin{equation}
\label{eq:complementarity_constraints}
    \lambda \geq 0, \; \phi(q,\lambda) \geq 0, \; \phi(q,\lambda)^T \lambda = 0,
\end{equation}
where the function $\phi:\mathbb{R}^p \times \mathbb{R}^m \rightarrow \mathbb{R}^m $ is a gap function which relates the distance between robot and object with the contact force. 
This complementarity framework is widespread within the robotics community and has been commonly used to simulate contact dynamics \cite{anitescu1997formulating, stewart2000implicit}, quasi-statics \cite{halm2019quasi}, leveraged in trajectory optimization \cite{posa2014direct}, stability \cite{haas2016distinction} and control \cite{posa2015stability} of rigid-body systems with contacts. Additionally, linear complementarity systems \cite{heemels2000linear} \cite{mayne2001control} capture the local behavior of \eqref{eq:manipulator} with the constraints \eqref{eq:complementarity_constraints}.
A linear complementarity system is characterized by the following five matrices: $A \in \mathbb{R}^{n \times n}$, $B \in \mathbb{R}^{n \times k}$, $D \in \mathbb{R}^{n \times m}$, $E \in \mathbb{R}^{m \times n}$, and $F \in \mathbb{R}^{m \times m}$ in the following way:
\begin{definition}(Linear Complementarity System)
\label{definition_lcs}
An LCS describes the evolution of the state trajectory $x = x(t) \in \mathbb{R}^n$ that are dependent on the contact forces $\lambda = \lambda(t) \in \mathbb{R}^m$ such that
\begin{equation}
\label{eq:LCS}
  \begin{aligned}
    & \quad \dot{x} = Ax+Bu+D\lambda,\\
    & 0 \leq \lambda \perp Ex +  F \lambda + c \geq 0,
  \end{aligned}
\end{equation}
where $u \in \mathbb{R}^k$ is the input vector, $A$ determines the autonomous dynamics of the state vector $x$, $B$ models the linear effect of the input on the state, $D$ describes the linear effect of the contact forces on the state.
\end{definition}
Note that the contact forces $\lambda$ are always non-negative. \eqref{eq:LCS} implies that either $\lambda = 0$ or $Ex + F \lambda + c = 0$, encoding the multi-modal dynamics of contact.

In many robotic scenarios, the hybrid, nonlinear dynamics of \eqref{eq:manipulator} can be locally approximated as a piecewise linear dynamical system, where each piece corresponds to a hybrid mode.
We note that the smooth components of the dynamics ($M(q),C(q,v),J(q), \phi(q)$, etc.) are linearized, but the LCS maintains the non-smooth nature of the original dynamics.
An LCS is a compact representation, as the variables and constraints scale linearly with $m$, rather than with the $2^m$ hybrid modes \cite{camlibel2001complementarity}, \cite{lin2009stability}.

\subsection{Linear Complementarity Problem}
Here, we recall some definitions and results from the theory of linear complementarity problems \cite{cottle2009linear}.
\begin{definition}(Linear Complementarity Problem)
Given $F \in \mathbb{R}^{m \times m}$ and a vector $w \in \mathbb{R}^m$, the $LCP(w,F)$  describes the following mathematical program:
\begin{alignat}{2}
\label{LCS_definiton}
\notag & \underset{}{\text{find}} && \lambda \in \mathbb{R}^m \\
& \text{subject to}  \quad && 0 \leq \lambda \perp F \lambda + w \geq 0.
\end{alignat}
\end{definition}
For a given $F$ and $w$, the LCP may have multiple solutions or none at all. Hence, we denote the solution set of the linear complementarity problem $\text{LCP}(w,F)$ as $\text{SOL}(w,F)$. We will restrict ourselves to a particular class of LCPs that are guaranteed to have unique solutions.
Note that without this restriction, contact dynamics are known to generate non-unique solutions \cite{Stewart2000, Halm2019}.
\begin{definition}(P-Matrix)
A matrix $F \in \mathbb{R}^{m \times m}$ is a P-matrix, if the determinant of all of its principal sub-matrices are positive; that is, $\text{det}(F_{\alpha \alpha}) \geq 0$ for all $\alpha \subseteq \{ 1, \ldots, m \} $.
\end{definition}
If $F$ is a P-matrix, then the solution set $\text{SOL}(w,F)$ is a singleton for any $w \in \mathbb{R}^m$ \cite{shen2005linear}. 
%
If we denote the unique element of $\text{SOL}(w,F)$ as $\psi(w)$, then $\psi(w)$ is a piecewise linear function in $w \in \mathbb{R}^m$, hence is Lipschitz continuous and directionally differentiable \cite{shen2005linear}.

We can describe the contact forces in \eqref{eq:LCS} as the solution to the $\text{LCP}(Ex + c,F)$, and denote $\lambda(x)$ as the unique element of $\text{SOL}(Ex + c, F)$. Therefore, $\lambda(x)$ is globally Lipschitz continuous and directionally differentiable \cite{camlibel2006lyapunov}. We denote the directional derivative of $\lambda$ in the direction $d$ as $\lambda'(x;d)$.
We now define the following sets that will be used throughout the work.
\begin{equation*}
    \Gamma_{\text{SOL}}(E,F,c) = \{ (x,\lambda) : \lambda \in SOL(Ex+c,F) \},
\end{equation*}
describes the graph of $\lambda(x)$.
\begin{equation*}
    \Gamma'_{\text{SOL}}(E,F,c,d) = \{ (x,\lambda,\lambda'(x;d)) : \lambda \in SOL(Ex+c,F) \},
\end{equation*}
describes the graph of $\Psi(x) = \begin{pmatrix} \lambda(x) \\ \lambda'(x;d) \end{pmatrix}$.

In the special case where $c=0$, we can describe $\lambda(x)$ and $\lambda(x;d)$ more precisely. We first define three index sets
\begin{equation*}
    \alpha(x) = \{ i : \lambda_i (x) > 0 = ( Ex + F \lambda (x) )_i \},
\end{equation*}
\begin{equation*}
    \beta(x) = \{ i : \lambda_i (x) = 0 = ( Ex + F \lambda (x) )_i \},
\end{equation*}
\begin{equation*}
    \gamma(x) = \{ i : \lambda_i (x) = 0 < ( Ex + F \lambda (x) )_i \},
\end{equation*}
where $\alpha$ indicates active contacts, $\gamma$ inactive contacts, and $\beta$ contacts that are inactive, but on the constraint boundary.
Using these index sets, it follows that the contact force $\lambda \in \mathbb{R}^m$ is equivalent to,
\begin{equation}
\label{eq:contact_force}
    \lambda_{\alpha}(x) = - ( F_{\alpha \alpha} )^{-1} E_{\alpha \bullet} x, \; \; \lambda_{\hat{\alpha}} (x) = 0,
\end{equation}
where $\alpha = \alpha(x)$ and $\hat{\alpha} = \beta(x) \cup \gamma(x)$. 
Similarly, there exists a subset $\beta_d \subseteq \beta(x)$ such that the directional derivative of the contact force $\lambda'(x;d)$ in the direction $d$ is given by \cite{camlibel2006lyapunov}
\begin{equation}
\label{eq:contact_force_derivative}
    \lambda'_{\alpha_d}(x;d) = - ( F_{\alpha_d \alpha_d} )^{-1} E_{\alpha_d \bullet} d, \; \; \lambda'_{\hat{\alpha}_d} (x) = 0,
\end{equation}
where $\alpha_d = \alpha(x) \cup \beta_d$ and $\hat{\alpha}_d = \bar{m} \setminus \alpha_d$.
Note the inclusion of elements from $\beta(x)$ in the directional derivative, as contacts in $\beta(x)$ might instantaneously become active.

Under the assumption that $F$ is a P-matrix, we now represent an LCS in a more compact manner.
The linear complementarity system in \eqref{eq:LCS} is equivalent to the dynamical system
\begin{align}
\label{eq:system_rep}
    &\dot{x} = Ax + Bu + D \lambda(x),
\end{align}
where $\lambda(x)$ corresponds to the unique element of $\text{SOL}(Ex+c,F)$ for every state vector $x$. Notice that \eqref{eq:system_rep} is only a more compact representation of \eqref{LCS_definiton} and still has the same structure as the LCS.

\section{Conditions for Stabilization}
In this section, we use the complementarity formulation of contact in order to construct conditions for stability in the sense of Lyapunov and propose controller design methods based on these conditions. Unlike the common approach of designing controllers only using state feedback (i.e., $u = u(x)$), we consider controllers of the form $u = u(x,\lambda)$ where the feedback is dependent both on the state $x$ and the contact force $\lambda$. Additionally, we exploit the structure of the complementarity system so our controller design avoids combinatorial mode enumeration.

To achieve this, we consider the input vector $u(x,\lambda) = Kx + L \lambda$, where $K \in \mathbb{R}^{k \times n}$ and $L \in \mathbb{R}^{k \times m}$ are feedback gain matrices on state and force, respectively. Note that both the controller $u(x,\lambda) = K x + L \lambda$ and the linear complementarity system \eqref{eq:LCS} are globally Lipschitz continuous under the P-matrix assumption.
Since the LCS in \eqref{eq:LCS} is Lipschitz continuous, it has unique solutions for any initial condition \cite{khalil2002nonlinear}. 

Given that solutions are unique, we can adopt standard notions  of stability for differential equations where the right-hand side is Lipschitz continuous, though possibly non-smooth \cite{camlibel2006lyapunov}, \cite{khalil2002nonlinear}. 

We now construct conditions for guaranteed stabilization. First, we consider the Lyapunov function introduced in \cite{camlibel2006lyapunov},
\begin{equation}
\label{eq:lyap_func}
    V(x,\lambda) = \begin{bmatrix} x^T & \lambda^T \end{bmatrix} \begin{bmatrix} P & Q \\ Q^T & R \end{bmatrix} \begin{bmatrix} x \\ \lambda \end{bmatrix},
\end{equation}
where $P \in \mathbb{R}^{n \times n}$, $Q \in \mathbb{R}^{n \times m}$,  and $R \in \mathbb{R}^{m \times m}$. Observe that the Lyapunov function \eqref{eq:lyap_func} is quadratic in terms of the pair $(x,\lambda)$, and piecewise quadratic in the state variable $x$ since the contact forces are a function of $x$ (i.e., $\lambda = \lambda(x)$). More precisely, $V$ is a piecewise quadratic Lyapunov function in $x$ where the function switches based on which contacts are active. For example, if all of the contact forces are inactive, $\lambda = 0$, then $V(x) = x^T P x$.

In Lyapunov based analysis and synthesis methods, one desires to search over a wide class of functions. Since $V(x,\lambda)$ is a piecewise quadratic Lyapunov function, it is more expressive than a Lyapunov function common to all modes (as was used in \cite{posa2015stability}), which makes it a more powerful choice than a single quadratic Lyapunov function \cite{johansson1997computation}. Additionally, observe that the $\lambda = \lambda(x)$ is a piecewise continuous function of the state vector $x$, and the function $V$ in \eqref{eq:lyap_func} is non-smooth.

Since both the function $V$ in \eqref{eq:lyap_func} and the contact force $\lambda$ are non-smooth, we work with directional derivatives in the direction $\dot{x}$ where $\dot{x} = Ax + Bu + D\lambda$. For simplicity, we denote $\lambda'(x;\dot{x}) = \bar{\lambda}(x)$. Now, we are ready to construct conditions for certifiably stabilizing gains $K$ and $L$.
\begin{theorem}
\label{theorem_stability}
Consider the linear complementarity system in \eqref{eq:LCS}. Assume that $F$ is a P-matrix and $c \geq 0$, then the input $u(x,\lambda) = Kx + L \lambda$ is a continuous function of $x$ and the LCS in \eqref{eq:LCS} is Lipschitz continuous. Furthermore, consider the matrices
\begin{equation*}
    M = \begin{bmatrix} P & Q \\ Q^T & R \end{bmatrix}, \;  N = \begin{bmatrix} N_{11} & N_{12} & Q \\ N_{12}^T & N_{22} & R \\ Q^T & R & 0 \end{bmatrix},
\end{equation*}
where $P \in \mathbb{R}^{n \times n}$, $Q \in \mathbb{R}^{n \times m}$, $R \in \mathbb{R}^{m \times m}$, and,
\begin{equation*}
    N_{11} = A^T P + PA + K^T B^T P + P B K,
\end{equation*}
\begin{equation*}
    N_{12} = P B L + PD + A^T Q + K^T B^T Q,
\end{equation*}
\begin{equation*}
    N_{22} = L^T B^T Q + D^T Q + Q^T D + Q^T B L.
\end{equation*}
If there exists $K$, $L$, $P$, $Q$, $R$ and a domain $\mathcal{D} \subseteq \mathbb{R}^n$ such that $V(x,\lambda) > 0$ for $(x, \lambda) \in \Gamma_{\text{SOL}}(E,F,c)$, $x \in \mathcal{D}$, then
\begin{enumerate}
	\item $x_e = 0$ is Lyapunov stable if $\hat{V}'(x;\dot{x}) \leq 0$ for \\ $(x, \lambda, \Bar{\lambda}) \in \Gamma'_{\text{SOL}}(E,F,c,\dot{x})$, $x \in \mathcal{D}$,
	\item $x_e = 0$ is asymptotically stable if $\hat{V}'(x;\dot{x})< 0$ for $(x, \lambda, \Bar{\lambda}) \in \Gamma'_{\text{SOL}}(E,F,c,\dot{x})$, $x \in \mathcal{D}$,
\end{enumerate}
where $z = [x \; \; \lambda \; \; \bar{\lambda}]^T$, $g = [x \; \; \lambda]^T$, $\dot{x} = Ax + Bu + D\lambda$, $V(x,\lambda) = g^T M g$, and $\hat{V}'(x;\dot{x}) = z^T N z$.
\end{theorem}
\begin{proof}Consider the Lyapunov function in \eqref{eq:lyap_func},
\begin{equation*}
    V(x,\lambda) = x^T P x + 2 x^T Q \lambda + \lambda^T R \lambda.
\end{equation*}
Then the composite function $\bar{V}(x)$ is locally Lipschitz continuous and directionally differentiable,
\begin{equation*}
    \bar{V}(x) = x^T P x + 2 x^T Q \lambda(x) + \lambda(x)^T R \lambda(x).
\end{equation*}
The directional derivative of $\bar{V}$ in the direction $\dot{x}$ can be shown as follows,
\begin{align*}
    \hat{V}'(x;\dot{x}) = 2 x^T P \dot{x} + &2 \dot{x}^T Q \lambda(x) + 2 x^T Q \bar{\lambda}(x)\\
     &+ 2 \lambda(x)^T R \bar{\lambda}(x),
\end{align*}
where $\dot{x} = (A+BK)x + (D+BL)\lambda$. After algebraic manipulation,
\begin{equation*}
    \hat{V}'(x;\dot{x}) = z^T N z,
\end{equation*}
and (a) and (b) follows.
\end{proof}

We have established sufficient conditions to stabilize the LCS in \eqref{eq:LCS}. Furthermore, observe that $V$ is a piecewise quadratic function, $u$ is a piecewise affine function and both of them switch based on active contacts. More precisely, $V$ and $u$ encode the non-smoothness of the problem structure, mirroring the structure of the LCS, and allow tactile feedback, without exponential enumeration. This is an appealing middle ground between the common Lyapunov function of our prior work \cite{posa2015stability}, and purely hybrid approaches \cite{Papachristodoulou09a, marcucciwarm}. We can assign a different Lyapunov function and a control policy for each mode but avoid mode enumeration so that the approach can scale to large number of contacts ($m$).


The sufficient conditions in Theorem \ref{theorem_stability} are matrix inequalities over the sets $\Gamma_{\text{SOL}}(E,F,c)$ and $\Gamma'_{\text{SOL}}(E,F,c,\dot{x})$. The challenge, then is to find representations of these sets suitable for control synthesis, without exponential enumeration of the hybrid modes. The set $\Gamma_{\text{SOL}}(E,F,c)$ can be explicitly written as:
\begin{equation*}
	\Gamma_{\text{SOL}}(E,F,c) = \{ (x,\lambda) : 0 \leq \lambda \perp Ex + c + F \lambda \geq 0 \}.
\end{equation*}
For the set $\Gamma'_{\text{SOL}}(E,F,c,\dot{x})$, we propose two conservative representations.
The challenge in representing $\Gamma'$ comes from properly bounding the directional derivative of force, $\bar \lambda$ without explicitly computing it (which would be combinatoric).
First, we consider the case where the constant term is set to zero ($c = 0$) and define the set
\begin{align}
\label{eq:sol_graph_derivative}
    \notag \Gamma'_{\text{SOL(1)}}(E,F,0,\dot{x}) &= \{ (x,\lambda,\bar{\lambda}) : 0 \leq F \lambda + Ex \perp \lambda \geq 0, \\
    & \quad ||\lambda||_2 \leq \gamma ||x||_2, ||\bar{\lambda}||_2 \leq \kappa ||x||_2 \},
\end{align}
where $\gamma$ and $\kappa$ are system-specific bounds on $\lambda$ and $\bar \lambda$, such that $\Gamma'_{\text{SOL}}(E,F,0,\dot{x}) \subseteq \Gamma'_{\text{SOL(1)}}(E,F,0,\dot{x})$.
The parameter $\gamma$ can be computed from \eqref{eq:contact_force}:
\begin{equation*}
    \gamma = \max_\alpha ||-( F_{\alpha \alpha} )^{-1} E_{\alpha \bullet}||_2.
\end{equation*}
Similarly, we can compute $\kappa$ using \eqref{eq:contact_force_derivative},
\begin{equation*}
    ||\bar{\lambda}||_2 \leq \gamma ||\dot{x}||_2 \leq \kappa ||x||_2,
\end{equation*}
with $\kappa = \gamma \big( ||A+BK||_2 + \gamma ||D+BL||_2 \big)$. 
In order to obtain bounds on $||A+BK||_2$ and $||D+BL||_2$, without yet knowing $K$ and $L$, we use the following proposition:
\begin{proposition}\cite{boyd2004convex}
\label{singular_value_bound}
The constraint $\sigma_{\text{max}}(A) \leq \Theta$ can be written as a convex linear matrix inequality. Let
\begin{equation*}
    C^{\Theta} = \{ A :  \begin{bmatrix} \Theta I & A \\ A^T & \Theta I \end{bmatrix} \succeq 0 \}.
\end{equation*}
Then, $A \in C^{\Theta} \Leftrightarrow A^T A \preceq \Theta^2 I \Leftrightarrow \sigma_{\text{max}}(A) \leq \Theta$.
\end{proposition}
Using Proposition \ref{singular_value_bound} with the equality $||A||_2 = \sigma_{\text{max}}(A)$, and given bounds on $K$ and $L$, we can obtain upper bounds on $||A+BK||_2$ and $||D+BL||_2$.
However, we note that this approach, while practically useful, requires enumerating all of the modes.
Hence, we propose a second approach that both avoids enumeration and addresses cases where the constant term is strictly positive ($c > 0$).
\begin{proposition}
\label{constraints_s_procedure}
Assume that $F$ is a P-matrix. If $\lambda_i > 0$, then $E_i^T \dot{x} + F_i^T \bar{\lambda} = 0$ for all $i$. Furthermore, using a slack variable $\rho \in \mathbb{R}^m$, this relation is captured by the constraints,
\begin{equation}
\label{eq:rho_lambda}
    \lambda_i \rho_i = 0 \; \; \text{for all} \; i \in \bar{m},
\end{equation}
\begin{equation}
\label{eq:rho_differentiation}
    E \dot{x} + F \bar{\lambda} + \rho = 0.
\end{equation}
\end{proposition}
\begin{proof}
Observe that if $\lambda_i > 0$, then
\begin{equation*}
    E_i^T x + F_i^T \lambda + c = 0.
\end{equation*}
$E_i^T \dot{x} + F_i^T \bar{\lambda} = 0$ follows after differentiation with respect to time on both sides of the equation. If $\lambda_i > 0$, then $\rho_i = 0$ from \eqref{eq:rho_lambda}, and \eqref{eq:rho_differentiation} is equal to $E_i^T \dot{x} + F_i^T \bar{\lambda} = 0$. If $\lambda_i = 0$, then \eqref{eq:rho_differentiation} holds trivially since $\rho_i$ is a free variable.
\end{proof}

We define the following set based on Proposition \ref{constraints_s_procedure}:
\begin{align}
\label{eq:sol_graph_derivative_2}
    \notag\Gamma'_{\text{SOL(2)}}&(E,F,c,\dot{x}) = \{ (x,\lambda,\bar{\lambda}) : 0 \leq F \lambda + Ex + c \perp \lambda \geq 0, \\
    & E \dot{x} + F \bar{\lambda} + \rho = 0, \; \lambda_i \rho_i = 0 \; \; \forall \; i \in \bar{m} \}.
\end{align}
Notice that $\Gamma'_{\text{SOL}}(E,F,c,\dot{x}) \subseteq \Gamma'_{\text{SOL(2)}}(E,F,c,\dot{x})$. Since we have defined the sets $\Gamma_{\text{SOL}}(E,F,c)$ and $\Gamma'_{\text{SOL(i)}}(E,F,c,\dot{x})$ for $i = 1,2$, we can formulate the feasibility problem:
\begin{alignat}{2}
\label{eq:feasability_LCS}
& \underset{}{\text{find}} && V(x,\lambda), K, L \\
\notag& \text{subject to}  \quad && V(0,0) = 0, \\
\notag& && V(x,\lambda) > 0, \quad \text{for} \; (x,\lambda) \in \mathcal{A}, x  \in \mathcal{D},  \\
\notag& && \hat{V}'(x;\dot{x}) \leq 0, \; \text{for} \; (x,\lambda,\bar{\lambda}) \in \mathcal{B}, x \in \mathcal{D},
\end{alignat}
where $\mathcal{A} = \Gamma_{\text{SOL}}(E,F,c)$, $\mathcal{B} = \Gamma'_{\text{SOL(i)}}(E,F,c,\dot{x})$ with $i$ either $1$ or $2$, with the functions $V(x,\lambda)$ and $\hat{V}'(x;\dot{x})$ as in Theorem \ref{theorem_stability}. Notice that $\hat{V}'(x;\dot{x}) \leq 0$ is a bilinear matrix inequality because of the bilinear terms $PBK$, $P B L$, $K^T B^T Q$ and $Q^T B L$. Therefore, we have formulated the problem of designing a control policy as finding a feasible solution for a set of bilinear matrix inequalities. 
The conditions that define $\mathcal{A}$ and $\mathcal{B}$ are incorporated via the S-procedure \cite{boyd1994linear}.

Even though the problem is a BMI, we can eliminate some of the bilinear constraints. Observe that the gain matrices $K$ and $L$ appear in the constraint $E \dot{x} + F \bar{\lambda} + \rho = 0$ in \eqref{eq:sol_graph_derivative_2}. This is undesirable in practice, as it will lead to bilinear S-procedure terms. 
To avoid this, we relax $E \dot{x} + F \bar{\lambda} + \rho = 0$ and bound $\bar \lambda$ with the inequality
\begin{equation}
\label{eq:bound_s_procedure}
||q||_2 \leq ||EBK||_2 ||x||_2 + \sigma ||EBL||_2 ||\lambda||_2 \leq b,
\end{equation}
where $q = EAx + ED\lambda + F \bar{\lambda} + \rho$. We can generate bounds on $||EBK||_2$ and $||EBL||_2$ using Proposition \ref{singular_value_bound}. In order to compute the final bound, $b$, we solve the following convex optimization problem,
\begin{alignat*}{2}
& \underset{}{\text{minimize}} && \sigma, d  \\
\notag& \text{subject to} \quad && \lambda^T \lambda \leq \sigma x^T x + d, \quad \text{for} \; (x,\lambda) \notag\in \Gamma_{\text{SOL}}(E,F,c),
\end{alignat*}
and compute $b = ||EBK||_2 \omega + ||EBL||_2 \sqrt{\gamma \omega^2 + d}$ considering the compact set $\mathcal{D} = \{ x : ||x||_2 \leq \omega \}$. Notice that \eqref{eq:bound_s_procedure} helps us avoid bilinear S-procedure terms, but does not turn the feasibility problem \eqref{eq:feasability_LCS} into a convex problem.

\section{Examples}
In this paper, we use the YALMIP \cite{Lofberg2004} toolbox with PENBMI \cite{kovcvara2003pennon} to formulate and solve bilinear matrix inequalities. SeDuMi \cite{sturm1999using} is used for solving the semidefinite programs (SDPs). PATH \cite{dirkse1995path} has been used to to solve the linear complementarity problems when performing simulations.

\subsection{Cart-Pole with Soft Walls}
We consider the cart-pole system where the goal is to balance the pole and regulate the cart to the center, where there are walls (modeled via spring contacts) on both sides.
This problem, or a slight variation of it, has been used as a benchmark in control through contact \cite{marcucciwarm}, \cite{deits2019lvis}, \cite{marcucci2017approximate} and the model is shown in Figure 1.

In our model, the $x_1$ is the position of the cart, $x_2$ is the angle of the pole, and $x_3$, $x_4$ are their time derivatives respectively. We can control the cart with the input $u_1$ and model the contact forces of the walls as $\lambda_1$ and $\lambda_2$, leading to the LCS
\begin{align*}
    \dot{x}_1& = x_3, \\
    \dot{x}_2& = x_4, \\
    \dot{x}_3 &= \frac{g m_p}{m_c} x_2 + \frac{1}{m_c} u_1,\\
    \quad \dot{x}_4 &= \frac{g (m_c + m_p)}{l m_c}  x_2 + \frac{1}{l m_c} u_1 + \frac{1}{l m_p} \lambda_1 - \frac{1}{l m_p} \lambda_2,\\
    \quad 0 &\leq \lambda_1 \perp l x_2 - x_1 + \dfrac{1}{k_1}  \lambda_1 + d \geq 0,\\
    0 &\leq \lambda_2 \perp x_1 - l x_2 + \dfrac{1}{k_2}  \lambda_2 + d \geq 0,
\end{align*}
where $k_1 = k_2 = 10$ are stiffness parameters of the soft walls, $g = 9.81$ is the gravitational acceleration, $m_p = 0.5$ is the mass of the pole, $m_c = 1$ is the mass of the cart, $l = 0.5$ is the length of the pole, and $d = 0.1$ represents where the walls are. For this model, we solve the feasibility problem \eqref{eq:feasability_LCS} and find a controller of the form $u(x,\lambda) = Kx + L \lambda$ that regulates the model to the origin. The algorithm succeeded in finding a feasible controller in 0.72 seconds. As a comparison, we also designed an LQR controller with penalty on the state $Q = 10 I$ and penalty on the input $R=1$. We tested both contact-aware and LQR controllers on the nonlinear plant for 100 initial conditions where $x_2(0)=0$, and $x_1(0), x_3(0), x_4(0)$ are uniformly distributed ($10x_1(0), x_3(0), x_4(0) \sim U[-1, 1]$). Despite the fact that the walls are not particularly stiff, LQR was successful only 81\% of the time, whereas our contact-aware policy was always successful.   

\begin{figure}[t]
	\label{fig:Cart-pole with soft walls}
	\includegraphics[width=1\columnwidth]{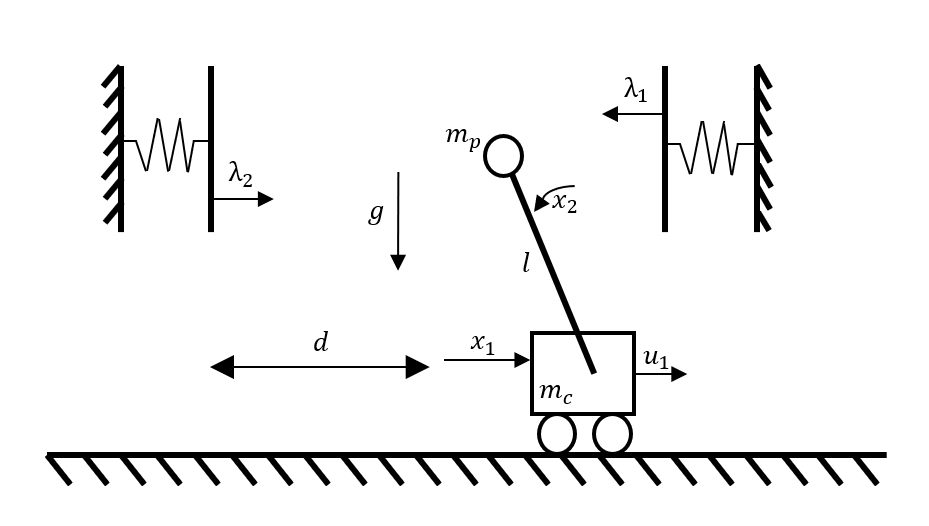}
	\caption{Benchmark problem: Regulation of the cart-pole system to the origin with soft walls.}
\end{figure}

\subsection{Partial State Feedback}

We consider a model that consists of three carts that are standing on a frictionless surface as in Figure 2. The cart on the left has a pole attached to it and the cart in the middle has two springs attached to it that represents soft contacts. In this model, a spring only becomes active if the distance between the outer block and the block in the middle is less than some threshold. Here, $x_1,x_2,x_3$ represent the positions of the carts and $x_4$ is the angle of the pole. 
The corresponding LCS is
\begin{align*}
    \ddot{x}_1 &= \frac{g m_p}{m_1} x_4 + \frac{1}{m_1} u_1 - \frac{1}{m_1} \lambda_1, \\
    \ddot{x}_2 &= \frac{\lambda_1}{m_2} - \frac{\lambda_2}{m_2}, \\
    \ddot{x}_3 &= \frac{\lambda_2}{m_3} + \frac{u_2}{m_3},\\
    \ddot{x}_4 &= \frac{g(m_1+m_p)}{m_1 l} x_4 + \frac{u_1}{m_1 l} - \frac{1}{m_1 l} \lambda_1,\\
    0 &\leq \lambda_1 \perp x_2 - x_1 + \frac{1}{k_1} \lambda_1 \geq 0,\\
    0 &\leq \lambda_2 \perp x_3 - x_2 + \frac{1}{k_2} \lambda_2 \geq 0,
\end{align*}
where the masses of the carts are $m_1 = m_2 = m_3 = 1$, $g = 9.81$ is the gravitational acceleration, $m_p = 1.5$ is the mass of the pole, $l = 0.5$ is the length of the pole, and $k_1 = k_2 = 100$ are stiffness parameters of the springs. Observe that we have control over the outer blocks, but do not have any control over the block in the middle. Additionally, we assume that we can not observe the middle block, and can only observe the outer blocks and the contact forces. For this example, we can solve the feasibility problem \eqref{eq:feasability_LCS} in 9.3 seconds and find a controller of the form $u(x,\lambda) = Kx + L \lambda$. We enforce sparsity on the controller $K$ to not use any feedback from the state $x_2$ or its derivative $\dot{x}_2$. This example demonstrates that tactile feedback can be used in scenarios where full state information is lacking. More precisely, we cannot observe the position and velocity of the cart in the middle, but we can design a controller that leverages tactile feedback and stabilize the system.

\begin{figure}[t]
	\includegraphics[width=1\columnwidth]{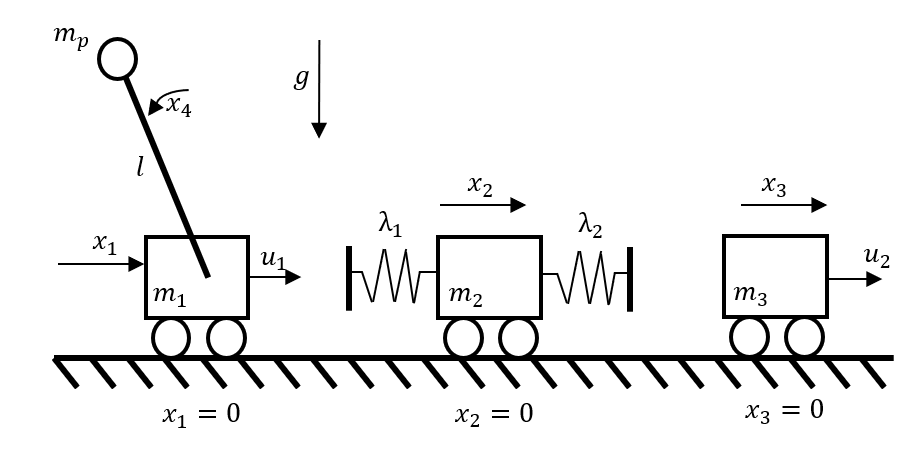}
	\caption{Regulation of carts to their respective origins without any state observations on the middle cart.}
\end{figure}

\subsection{Acrobot with Soft Joint Limits}
As a third example, we consider the classical underactuated acrobot, a double pendulum with a single actuator at the elbow (see \cite{murray1991case} for the details of the acrobot dynamics). Additionally, we add soft joint limits to the model. Hence we consider the model in Figure 3:
\begin{equation*}
	\dot{x} = A x + Bu + D \lambda,
\end{equation*}
where $x = (\theta_1, \theta_2, \dot{\theta}_1,\dot{\theta}_2)$, $\lambda = (\lambda_1, \lambda_2)$, and $D = \begin{bmatrix} 0_{2 \times 2} \\ M^{-1} J^T \end{bmatrix}$ with $J^T = \begin{bmatrix} -1 & 1 \\ 0 & 0 \end{bmatrix}$. For this model, the masses of the rods are $m_1 = 0.5$, $m_2 = 1$, the lengths of the rods are $l_1 = 0.5$, $l_2 = 1$, and the gravitational acceleration is $g = 9.81$. We model the soft actuation joint limit using the following complementarity constraints:
\begin{align*}
	&0 \leq d - \theta_1 + \dfrac{1}{k} \lambda_1 \perp \lambda_1 \geq 0,\\
	&0 \leq \theta_1 + d + \dfrac{1}{k} \lambda_2 \perp \lambda_2 \geq 0,
\end{align*}
where $k = 1$ is the stiffness parameter and $d = 1$ is the angle that represents the joint limits in terms of the angle $\theta_1$. For this example, we solve the feasibility problem \eqref{eq:feasability_LCS} and obtain a controller of the form $u(x,\lambda) = Kx + L \lambda$ in 1.18 seconds. For comparison, we also designed an LQR controller for the linear system where the penalty on the state is $Q=100I$ and the penalty on the input is $R=1$. We ran 100 trials on the nonlinear plant where initial conditions were sampled according to $x_1(0),x_2(0) \sim U[-0.05, 0.05]$, $x_3(0) \sim U[-0.2, 0.2]$ and $x_4(0) \sim U[-0.1, 0.1]$. Out of these 100 trials, LQR was successful only 29\% of the time whereas our design was successful 68\% of the time.

\begin{figure}[t]
	\includegraphics[width=1\columnwidth]{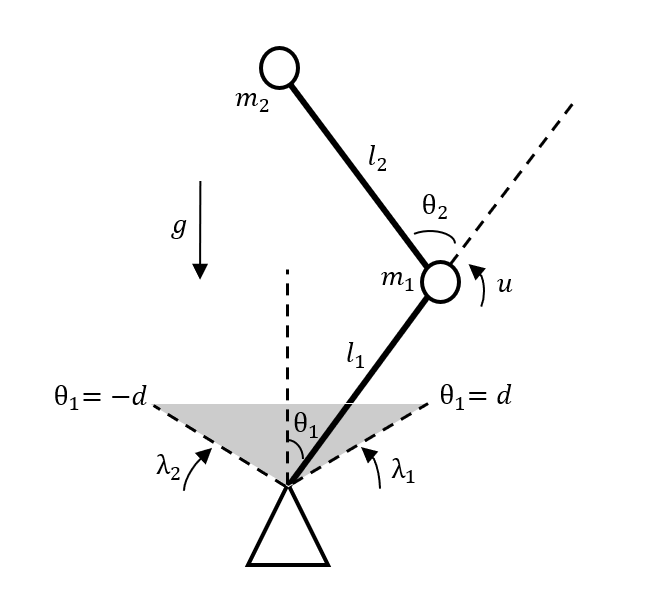}
	\caption{Acrobot with soft joint limits.}
\end{figure}

\section{Conclusion}
In this work, we have proposed an algorithm for synthesizing control policies that utilize both state and force feedback. We have shown that pure local, linear analysis was entirely insufficient and utilizing contact in the control design is critical to achieve high performance. Furthermore, the proposed algorithm exploits the complementarity structure of the system and avoids enumerating the exponential number of potential modes, enabling efficient design of multi-contact control policies. In addition to incorporating tactile sensing into dynamic feedback, we provide stability guarantees for our design method.

The algorithm requires solving feasibility problems that include bilinear matrix inequalities and we have used PENBMI \cite{kovcvara2003pennon} in this work. For the examples that are presented here, the runtime of the algorithm was short and we found solutions to the problems relatively quickly. On the other hand, it is important to note that for some parameter choices and initializations, the solver was unable to produce feasible solutions.

The current framework requires the contact forces to be continuous, and thus is limited to soft contact models. In future work, we will to extend our results to include impulsive impacts in the dynamics alongside continuous forces and tactile feedback.
We will also consider extensions to systems where there is a direct relation between the contact force and the input. For example,  quasistatic pushing and grasping  can be modeled as an LCS where the input $u$ also appears in the complementarity constraints \cite{halm2019quasi}.
These systems, when combined with our control policies, would create an algebraic loop that must be broken via introduction of delay or filtering (modeling sensor dynamics).


\addtolength{\textheight}{-1cm}  

\bibliographystyle{ieeetr}
\bibliography{Refs}

\begin{thebibliography}{10}

\bibitem{Tassa10}
Y.~Tassa and E.~Todorov, ``{Stochastic Complementarity for Local Control of
  Discontinuous Dynamics},'' in {\em Proceedings of Robotics: Science and
  Systems (RSS)}, Citeseer, 2010.

\bibitem{Hogan2017}
F.~R. Hogan, E.~R. Grau, and A.~Rodriguez, ``{Reactive Planar Manipulation with
  Convex Hybrid MPC},'' Oct 2017.

\bibitem{marcucci2017approximate}
T.~Marcucci, R.~Deits, M.~Gabiccini, A.~Bicchi, and R.~Tedrake, ``Approximate
  {H}ybrid {M}odel {P}redictive {C}ontrol for {M}ulti-contact {P}ush {R}ecovery
  in {C}omplex {E}nvironments,'' in {\em 2017 IEEE-RAS 17th International
  Conference on Humanoid Robotics (Humanoids)}, pp.~31--38, IEEE, 2017.

\bibitem{posa2015stability}
M.~Posa, M.~Tobenkin, and R.~Tedrake, ``Stability {A}nalysis and {C}ontrol of
  {R}igid-body {S}ystems with {I}mpacts and {F}riction,'' {\em IEEE
  Transactions on Automatic Control}, vol.~61, no.~6, pp.~1423--1437, 2015.

\bibitem{wettels2009multi}
N.~Wettels, J.~Fishel, Z.~Su, C.~Lin, and G.~Loeb, ``Multi-modal {S}ynergistic
  {T}actile {S}ensing,'' in {\em Tactile {S}ensing in {H}umanoids {-} {T}actile
  {S}ensors and {B}eyond {W}orkshop, 9th IEEE-RAS International Conference on
  Humanoid Robots}, 2009.

\bibitem{guggenheim2017robust}
J.~W. Guggenheim, L.~P. Jentoft, Y.~Tenzer, and R.~D. Howe, ``Robust and
  {I}nexpensive {S}ix-axis {F}orce--torque {S}ensors {U}sing {M}{E}{M}{S}
  {B}arometers,'' {\em IEEE/ASME Transactions on Mechatronics}, vol.~22, no.~2,
  pp.~838--844, 2017.

\bibitem{yuan2015measurement}
W.~Yuan, R.~Li, M.~A. Srinivasan, and E.~H. Adelson, ``Measurement of {S}hear
  and {S}lip with a {G}el{S}ight {T}actile {S}ensor,'' in {\em 2015 IEEE
  International Conference on Robotics and Automation (ICRA)}, pp.~304--311,
  IEEE, 2015.

\bibitem{kumar2016optimal}
V.~Kumar, E.~Todorov, and S.~Levine, ``Optimal {C}ontrol with {L}earned {L}ocal
  {M}odels: Application to {D}exterous {M}anipulation,'' in {\em 2016 IEEE
  International Conference on Robotics and Automation (ICRA)}, pp.~378--383,
  IEEE, 2016.

\bibitem{donlon2018gelslim}
E.~Donlon, S.~Dong, M.~Liu, J.~Li, E.~Adelson, and A.~Rodriguez, ``Gel{S}lim: A
  {H}igh-{R}esolution, {C}ompact, {R}obust, and {C}alibrated {T}actile-sensing
  {F}inger,'' in {\em 2018 IEEE/RSJ International Conference on Intelligent
  Robots and Systems (IROS)}, pp.~1927--1934, IEEE, 2018.

\bibitem{howe1993tactile}
R.~D. Howe, ``Tactile {S}ensing and {C}ontrol of {R}obotic {M}anipulation,''
  {\em Advanced Robotics}, vol.~8, no.~3, pp.~245--261, 1993.

\bibitem{romano2011human}
J.~M. Romano, K.~Hsiao, G.~Niemeyer, S.~Chitta, and K.~J. Kuchenbecker,
  ``Human-inspired {R}obotic {G}rasp {C}ontrol with {T}actile {S}ensing,'' {\em
  IEEE Transactions on Robotics}, vol.~27, no.~6, pp.~1067--1079, 2011.

\bibitem{yamaguchi2016combining}
A.~Yamaguchi and C.~G. Atkeson, ``Combining {F}inger {V}ision and {O}ptical
  {T}actile {S}ensing: {R}educing and {H}andling {E}rrors {W}hile {C}utting
  {V}egetables,'' in {\em 2016 IEEE-RAS 16th International Conference on
  Humanoid Robots (Humanoids)}, pp.~1045--1051, IEEE, 2016.

\bibitem{merzic2018leveraging}
H.~Merzic, M.~Bogdanovic, D.~Kappler, L.~Righetti, and J.~Bohg, ``Leveraging
  {C}ontact {F}orces for {L}earning to {G}rasp,'' {\em arXiv preprint
  arXiv:1809.07004}, 2018.

\bibitem{tian2019manipulation}
S.~Tian, F.~Ebert, D.~Jayaraman, M.~Mudigonda, C.~Finn, R.~Calandra, and
  S.~Levine, ``Manipulation by {F}eel: Touch-based {C}ontrol with {D}eep
  {P}redictive {M}odels,'' {\em arXiv preprint arXiv:1903.04128}, 2019.

\bibitem{camlibel2001complementarity}
M.~Camlibel {\em et~al.}, ``Complementarity {M}ethods in the {A}nalysis of
  {P}iecewise {L}inear {D}ynamical {S}ystems,'' tech. rep., Tilburg University,
  School of Economics and Management, 2001.

\bibitem{alur2000discrete}
R.~Alur, T.~A. Henzinger, G.~Lafferriere, and G.~J. Pappas, ``Discrete
  {A}bstractions of {H}ybrid {S}ystems,'' {\em Proceedings of the IEEE},
  vol.~88, no.~7, pp.~971--984, 2000.

\bibitem{branicky1998unified}
M.~S. Branicky, V.~S. Borkar, and S.~K. Mitter, ``A {U}nified {F}ramework for
  {H}ybrid {C}ontrol: {M}odel and {O}ptimal {C}ontrol {T}heory,'' {\em IEEE
  Transactions on Automatic Control}, vol.~43, no.~1, pp.~31--45, 1998.

\bibitem{cottle2009linear}
R.~W. Cottle, {\em Linear {C}omplementarity {P}roblem}.
\newblock Springer, 2009.

\bibitem{anitescu1997formulating}
M.~Anitescu and F.~A. Potra, ``Formulating {D}ynamic {M}ulti-rigid-body
  {C}ontact {P}roblems with {F}riction as {S}olvable {L}inear {C}omplementarity
  {P}roblems,'' {\em Nonlinear Dynamics}, vol.~14, no.~3, pp.~231--247, 1997.

\bibitem{stewart2000implicit}
D.~Stewart and J.~C. Trinkle, ``An {I}mplicit {T}ime-stepping {S}cheme for
  {R}igid {B}ody {D}ynamics with {C}oulomb {F}riction,'' in {\em Proceedings
  2000 ICRA. Millennium Conference. IEEE International Conference on Robotics
  and Automation. Symposia Proceedings (Cat. No. 00CH37065)}, vol.~1,
  pp.~162--169, IEEE, 2000.

\bibitem{halm2019quasi}
M.~Halm and M.~Posa, ``A {Q}uasi-static {M}odel and {S}imulation {A}pproach for
  {P}ushing, {G}rasping, and {J}amming,'' {\em arXiv preprint
  arXiv:1902.03487}, 2019.

\bibitem{posa2014direct}
M.~Posa, C.~Cantu, and R.~Tedrake, ``A {D}irect {M}ethod for {T}rajectory
  {O}ptimization of {R}igid {B}odies {T}hrough {C}ontact,'' {\em The
  International Journal of Robotics Research}, vol.~33, no.~1, pp.~69--81,
  2014.

\bibitem{haas2016distinction}
M.~Haas-Heger, G.~Iyengar, and M.~Ciocarlie, ``On the {D}istinction {B}etween
  {A}ctive and {P}assive {R}eaction in {G}rasp {S}tability {A}nalysis⋆,'' in
  {\em Workshop on the Algorithmic Foundation of Robotics (WAFR)}, 2016.

\bibitem{heemels2000linear}
W.~Heemels, J.~M. Schumacher, and S.~Weiland, ``Linear {C}omplementarity
  {S}ystems,'' {\em SIAM Journal on Applied Mathematics}, vol.~60, no.~4,
  pp.~1234--1269, 2000.

\bibitem{mayne2001control}
D.~Q. Mayne, ``Control of {C}onstrained {D}ynamic {S}ystems,'' {\em European
  Journal of Control}, vol.~7, no.~2-3, pp.~87--99, 2001.

\bibitem{lin2009stability}
H.~Lin and P.~J. Antsaklis, ``Stability and {S}tabilizability of {S}witched
  {L}inear {S}ystems: {A} {S}urvey of {R}ecent {R}esults,'' {\em IEEE
  Transactions on Automatic Control}, vol.~54, no.~2, pp.~308--322, 2009.

\bibitem{Stewart2000}
D.~E. Stewart, {\em {Rigid-Body Dynamics with Friction and Impact}}, vol.~42.
\newblock Jan 2000.

\bibitem{Halm2019}
M.~Halm and M.~Posa, ``{Modeling and Analysis of Non-unique Behaviors in
  Multiple Frictional Impacts},'' in {\em Robotics: Science and Systems}, 2019.

\bibitem{shen2005linear}
J.~Shen and J.-S. Pang, ``Linear {C}omplementarity {S}ystems: Zeno {S}tates,''
  {\em SIAM Journal on Control and Optimization}, vol.~44, no.~3,
  pp.~1040--1066, 2005.

\bibitem{camlibel2006lyapunov}
M.~K. Camlibel, J.-S. Pang, and J.~Shen, ``Lyapunov {S}tability of
  {C}omplementarity and {E}xtended {S}ystems,'' {\em SIAM Journal on
  Optimization}, vol.~17, no.~4, pp.~1056--1101, 2006.

\bibitem{khalil2002nonlinear}
H.~K. Khalil, ``Nonlinear {S}ystems,'' {\em Upper Saddle River}, 2002.

\bibitem{johansson1997computation}
M.~Johansson and A.~Rantzer, ``Computation of {P}iecewise {Q}uadratic
  {L}yapunov {F}unctions for {H}ybrid {S}ystems,'' in {\em 1997 European
  Control Conference (ECC)}, pp.~2005--2010, IEEE, 1997.

\bibitem{Papachristodoulou09a}
A.~Papachristodoulou and S.~Prajna, ``{Robust Stability Analysis of Nonlinear
  Hybrid Systems},'' {\em IEEE Transactions on Automatic Control}, vol.~54,
  pp.~1035--1041, May 2009.

\bibitem{marcucciwarm}
T.~Marcucci and R.~Tedrake, ``Warm {S}tart of {M}ixed-{I}nteger {P}rograms for
  {M}odel {P}redictive {C}ontrol of {H}ybrid {S}ystems,'' 2019.

\bibitem{boyd2004convex}
S.~Boyd and L.~Vandenberghe, {\em Convex {O}ptimization}.
\newblock Cambridge University Press, 2004.

\bibitem{boyd1994linear}
S.~Boyd, L.~El~Ghaoui, E.~Feron, and V.~Balakrishnan, {\em Linear {M}atrix
  {I}nequalities in {S}ystem and {C}ontrol {T}heory}, vol.~15.
\newblock S{I}{A}{M}, 1994.

\bibitem{Lofberg2004}
J.~L{\"{o}}fberg, ``Y{A}{L}{M}{I}{P} : A {T}oolbox for {M}odeling and
  {O}ptimization in {M}{A}{T}{L}{A}{B},'' in {\em In Proceedings of the CACSD
  Conference}, (Taipei, Taiwan), 2004.

\bibitem{kovcvara2003pennon}
M.~Ko{\v{c}}vara and M.~Stingl, ``P{E}{N}{N}{O}{N}: A {C}ode for {C}onvex
  {N}onlinear and {S}emidefinite {P}rogramming,'' {\em Optimization Methods and
  Software}, vol.~18, no.~3, pp.~317--333, 2003.

\bibitem{sturm1999using}
J.~F. Sturm, ``Using {S}e{D}u{M}i 1.02, {A} {M}{A}{T}{L}{A}{B} {T}oolbox for
  {O}ptimization over {S}ymmetric {C}ones,'' {\em Optimization Methods and
  Software}, vol.~11, no.~1-4, pp.~625--653, 1999.

\bibitem{dirkse1995path}
S.~P. Dirkse and M.~C. Ferris, ``The {P}{A}{T}{H} {S}olver: {A} {N}on-monotone
  {S}tabilization {S}cheme for {M}ixed {C}omplementarity {P}roblems,'' {\em
  Optimization Methods and Software}, vol.~5, no.~2, pp.~123--156, 1995.

\bibitem{deits2019lvis}
R.~Deits, T.~Koolen, and R.~Tedrake, ``L{V}{I}{S}: {L}earning {F}rom {V}alue
  {F}unction {I}ntervals for {C}ontact-aware {R}obot {C}ontrollers,'' in {\em
  2019 International Conference on Robotics and Automation (ICRA)},
  pp.~7762--7768, IEEE, 2019.

\bibitem{murray1991case}
R.~M. Murray and J.~E. Hauser, {\em A {C}ase {S}tudy in {A}pproximate
  {L}inearization: The {A}crobat {E}xample}.
\newblock 1991.

\end{thebibliography}

\end{document}